\documentclass[final]{article}

\usepackage[utf8]{inputenc} %
\usepackage[T1]{fontenc}    %

\usepackage[nonatbib]{neurips_2021}
\usepackage{times}
\usepackage{epsfig}
\usepackage{graphicx}
\usepackage{amsmath}
\usepackage{amssymb}
\usepackage{amsthm}
\newtheorem{theorem}{Theorem}
\newtheorem{remark}{Remark}
\usepackage{xcolor}
\usepackage{booktabs}
\usepackage{graphicx}
\usepackage{amsmath}
\usepackage{wrapfig}
\usepackage{caption}
\usepackage{diagbox}
\usepackage{multirow}
\usepackage[pagebackref=true,breaklinks=true,colorlinks,bookmarks=false,citecolor=blue,linkcolor=blue]{hyperref}

\usepackage{amsmath,amsfonts,bm}

\def\1{\bm{1}}

\def\vs{{\bm{s}}}

\DeclareMathAlphabet{\mathsfit}{\encodingdefault}{\sfdefault}{m}{sl}
\SetMathAlphabet{\mathsfit}{bold}{\encodingdefault}{\sfdefault}{bx}{n}

\renewcommand{\epsilon}{\varepsilon}

\newcommand{\be}{\mathbf{e}}

\newcommand{\bu}{\mathbf{u}}

\newcommand{\bx}{\mathbf{x}}
\newcommand{\by}{\mathbf{y}}

\newcommand{\bA}{\mathbf{A}}

\newcommand{\bD}{\mathbf{D}}

\newcommand{\bH}{\mathbf{H}}
\newcommand{\bI}{\mathbf{I}}

\newcommand{\bL}{\mathbf{L}}

\newcommand{\bR}{\mathbf{R}}
\newcommand{\bS}{\mathbf{S}}
\newcommand{\bT}{\mathbf{T}}
\newcommand{\bU}{\mathbf{U}}

\newcommand{\bW}{\mathbf{W}}
\newcommand{\bX}{\mathbf{X}}
\newcommand{\bY}{\mathbf{Y}}

\newcommand{\bvarepsilon}{\bm{\varepsilon}}

\newcommand{\bLambda}{\bm{\Lambda}}

\newcommand{\bTheta}{\bm{\Theta}}

\newcommand{\cG}{\mathcal{G}}

\newcommand{\cN}{\mathcal{N}}

\newcommand{\cV}{\mathcal{V}}

\newcommand{\cX}{\mathcal{X}}

\newcommand{\bbE}{\mathbb{E}}

\newcommand{\bbR}{\mathbb{R}}

\newcommand{\bzero}{\mathbf{0}}

\DeclareMathOperator*{\diag}{diag}

\def\ie{\textit{i.e.,~}}
\def\eg{\textit{e.g.,~}}

\def\etal{\textit{et al.~}}

\def\vs{\textit{v.s.~}}

\def\sota{state-of-the-art~}

\title{Dissecting the Diffusion Process in\\
Linear Graph Convolutional Networks}

\usepackage[symbol]{footmisc}

\author{%
  Yifei Wang$^1$ \quad Yisen Wang$^2$\thanks{Corresponding author: Yisen Wang (yisen.wang@pku.edu.cn).} \quad Jiansheng Yang$^1$ \quad Zhouchen Lin$^{2,3}$ \\
   $^1$ School of Mathematical Sciences, Peking University, China \\
   $^2$ Key Lab. of Machine Perception, School of EECS, Peking Univesity, China \\
   $^3$ Pazhou Lab, Guangzhou 510330, China \\
   \texttt{yifei\_wang@pku.edu.cn, yisen.wang@pku.edu.cn} \\ \texttt{yjs@math.pku.edu.cn, zlin@pku.edu.cn}
}

\begin{document}
\maketitle

\begin{abstract}
Graph Convolutional Networks (GCNs) have attracted more and more attentions in recent years. A typical GCN layer consists of a linear feature propagation step and a nonlinear transformation step. Recent works show that a linear GCN can achieve comparable performance to the original non-linear GCN while being much more computationally efficient. In this paper, we dissect the feature propagation steps of linear GCNs from a perspective of continuous graph diffusion, and analyze why linear GCNs fail to benefit from more propagation steps. Following that, we propose Decoupled Graph Convolution (DGC) that decouples the terminal time and the feature propagation steps, making it more flexible and capable of exploiting a very large number of feature propagation steps. Experiments demonstrate that our proposed DGC improves linear GCNs by a large margin and makes them competitive with many modern variants of non-linear GCNs. Code is available at \url{https://github.com/yifeiwang77/DGC}.

\end{abstract}

\section{Introduction}
Recently, Graph Convolutional Networks (GCNs) have successfully extended the powerful representation learning ability of modern Convolutional Neural Networks (CNNs) to the graph data \cite{kipf2016semi}. A graph convolutional layer typically consists of two stages: linear feature propagation and non-linear feature transformation. 
Simple Graph Convolution (SGC) \cite{wu2019simplifying} simplifies GCNs by removing the nonlinearities between GCN layers and collapsing the resulting function into a single linear transformation, which is followed by a single linear classification layer and then becomes a linear GCN. SGC can achieve comparable performance to canonical GCNs while being much more computationally efficient and using significantly fewer parameters. 
Thus, we mainly focus on linear GCNs in this paper.

Although being comparable to canonical GCNs, SGC 
still suffers from a similar issue as non-linear GCNs, that is, more (linear) feature propagation steps $K$ will degrade the performance catastrophically. This issue is widely characterized as the ``over-smoothing'' phenomenon. Namely, node features become smoothed out and indistinguishable after too many feature propagation steps \cite{li2018deeper}.

In this work, through a dissection of the diffusion process of linear GCNs, we characterize a fundamental limitation of SGC. Specifically,  we point out that its feature propagation step amounts to a very coarse finite difference with a fixed step size $\Delta t=1$, which results in a large numerical error. And because the step size is fixed, more feature propagation steps will inevitably lead to a large terminal time $T=K\cdot\Delta t\to\infty$ that over-smooths the node features.

To address these issues, we propose Decoupled Graph Convolution (DGC) by decoupling the terminal time $T$ and propagation steps $K$. In particular, we can flexibly choose a continuous terminal time $T$ for the optimal tradeoff between under-smoothing and over-smoothing, and then fix the terminal time while adopting more propagation steps $K$. In this way, different from SGC that over-smooths with more propagation steps, our proposed DGC can obtain a more fine-grained finite difference approximation with more propagation steps, which contributes to the final performance both theoretically and empirically. Extensive experiments show that DGC (as a linear GCN) improves over SGC significantly and obtains state-of-the-art results that are comparable to many modern non-linear GCNs. Our main contributions are summarized as follows:
\begin{itemize}
    \item We investigate SGC by dissecting its diffusion process from a continuous perspective, and characterize why it cannot benefit from more propagation steps. 
    \item We propose Decoupled Graph Convolution (DGC) that decouples the terminal time $T$ and the propagation steps $K$, which enables us to choose a continuous terminal time flexibly while benefiting from more propagation steps from both theoretical and empirical aspects. 
    \item Experiments show that DGC outperforms canonical GCNs significantly and obtains state-of-the-art (SOTA) results among linear GCNs, which is even comparable to many competitive non-linear GCNs. We think DGC can serve as a strong baseline for the future research. 
\end{itemize}

\section{Related Work}

\textbf{Graph convolutional networks (GCNs).} To deal with non-Euclidean graph data, GCNs are proposed for direct convolution operation over graph, and have drawn interests from various domains. GCN is firstly introduced for a spectral perspective \cite{yang2016revisiting,kipf2016semi}, but soon it becomes popular as a general message passing algorithm in the spatial domain. Many variants have been proposed to improve its performance, such as GraphSAGE \cite{hamilton2017inductive} with LSTM and GAT with attention mechanism \cite{velivckovic2018graph}. 

\textbf{Over-smoothing issue.}
GCNs face a fundamental problem compared to standard CNNs, \textit{i.e.}, the over-smoothing problem. Li \etal\cite{li2018deeper} offer a theoretical characterization of over-smoothing based on linear feature propagation. After that, many researchers have tried to incorporate effective mechanisms in CNNs to alleviate over-smoothing. DeepGCNs \cite{li2019deepgcns} shows that residual connection and dilated convolution can make GCNs go as deep as CNNs, although increased depth does not contribute much. Methods like APPNP \cite{klicpera2018predict} and JKNet \cite{xu2018representation} avoid over-smoothing by aggregating multi-scale information from the first hidden layer. DropEdge \cite{rong2019dropedge}  applies dropout to graph edges and find it enables training GCNs with more layers. PairNorm \cite{zhao2019pairnorm} regularizes the feature distance to be close to the input distance, which will not fail catastrophically but still decrease with more layers. 

\textbf{Continuous GCNs.} Deep CNNs have been widely interpreted from a continuous perspective, \eg ResNet \cite{he2016deep} as the Euler discretization of Neural ODEs \cite{lu18d,chen2018neural}. This viewpoint has recently been borrowed to understand and improve GCNs. GCDE \cite{poli2019graph} directly extends GCNs to a Neural ODE, while CGNN \cite{xhonneux2020continuous} devises a GCN variant inspired by a new continuous diffusion. Our method is also inspired by the connection between discrete and continuous graph diffusion, but alternatively, we focus on their numerical gap and characterize how it affects the final performance. 

{
\textbf{Linear GCNs.} SGC \cite{wu2019simplifying} simplifies and separates the two stages of GCNs: feature propagation and (non-linear) feature transformation. It finds that utilizing only a simple logistic regression after feature propagation (removing the non-linearities), which makes it a linear GCN, can obtain comparable performance to canonical GCNs.  In this paper, we further show that a properly designed linear GCN (DGC) can be on-par with state-of-the-art non-linear GCNs while possessing many desirable properties. For example, as a linear model, DGC requires much fewer parameters than non-linear GCNs, which makes it very memory efficient, and meanwhile, its training is also much faster ($\sim100\times$) than non-linear models as it could preprocess all features before training.
}

\section{Dissecting Linear GCNs from Continuous Dynamics}

In this section, we make a brief review of SGC \cite{wu2019simplifying} in the context of semi-supervised node classification task, and further point out its fundamental limitations. 

\subsection{Review of Simple Graph Convolution (SGC)}

Define a graph as $\cG=(\cV,\bA)$, where $\cV=\{v_1,\dots,v_n\}$ denotes the vertex set of $n$ nodes, and $\bA\in\bbR^{n\times n}$ is an adjacency matrix where $a_{ij}$ denotes the edge weight between node $v_i$ and $v_j$. The degree matrix $\bD=\diag(d_1,\dots,d_n)$  of $\bA$ is a diagonal matrix with its $i$-th diagonal entry as $d_i=\sum_{j}a_{ij}$. Each node $v_i$ is represented by a $d$-dimensional feature vector $\bx_i\in\bbR^d$, and we denote the feature matrix as $\bX\in\bbR^{n\times d}=[\bx_1,\dots,\bx_n]$. Each node belongs to one out of $C$ classes, denoted by a one-hot vector $\by_i\in\{0,1\}^C$. In node classification problems, only a subset of nodes $\cV_l\subset \cV$ are labeled and we want to predict the labels of the rest nodes $\cV_u=\cV\backslash\cV_l$.

SGC shows that we can obtain similar performance with a simplified GCN,
\begin{equation}
\hat{\mathbf{Y}}_{\mathrm{SGC}}=\operatorname{softmax}\left(\mathbf{S} ^{K} \mathbf{X} \boldsymbol{\Theta}\right),
\label{eqn:sgc-whole}
\end{equation}
which pre-processes the node features $\bX$ with $K$ linear propagation steps, and then applies a linear classifier with parameter $\bTheta$. Specifically, at the step $k$, each feature $\bx_i$ is computed by aggregating features in its local neighborhood,
which can be done in parallel over the whole graph for $K$ steps, 
\begin{equation}
    \mathbf{X}^{(k)} \leftarrow\bS \bX^{(k-1)}, \text{ where }     \mathbf{S} =\widetilde{\mathbf{D}}^{-\frac{1}{2}} \widetilde{\mathbf{A}} \widetilde{\mathbf{D}}^{-\frac{1}{2}}
    \quad\Longrightarrow\quad 
    \bX^{(K)}=\bS^{K} \bX.   
    \label{eqn:sgc-propagation}
\end{equation}
Here $\widetilde{\mathbf{A}}=\mathbf{A}+\mathbf{I}$ is the adjacency matrix augmented with the self-loop $\bI$, $\widetilde{\mathbf{D}}$ is the degree matrix of $\widetilde{\mathbf{A}}$, and $\mathbf{S} $ denotes the symmetrically normalized adjacency matrix. 
This step exploits the local graph structure to smooth out the noise in each node. 

At last, SGC applies a multinomial logistic regression (\emph{a.k.a.}~softmax regression) with parameter $\bTheta$ to predict the node labels $\hat\bY_{\rm SGC}$ from the node features of the last propagation step $\bX^{(K)}$:
\begin{equation}
    \hat{\mathbf{Y}}_{\mathrm{SGC}}=\operatorname{softmax}\left(\bX^{(K)}\mathbf{\Theta}\right).
\end{equation}
Because both the feature propagation $(\bS^{K}_{\rm{}}\bX)$ and classification ($\bX^{(K)}\bTheta$) steps are linear, SGC is essentially a linear version of GCN that only relies on linear features from the input.

\subsection{Equivalence between SGC and Graph Heat Equation}
\label{sec:dissection}

Previous analysis of linear GCNs focuses on their asymptotic behavior as propagation steps $K\to\infty$ (discrete), known as the over-smoothing phenomenon \cite{li2018deeper}. In this work, we instead provide a novel \emph{non-asymptotic} characterization of linear GCNs from the corresponding \emph{continuous dynamics}, graph heat equation \cite{chung1997spectral}. 
A key insight is that we notice that the propagation of SGC can be seen equivalently as a (coarse) numerical discretization of the graph diffusion equation, as we show below.

Graph Heat Equation (GHE) is a well-known generalization of the heat equation on graph data, which is widely used to model graph dynamics with applications in spectral graph theory \cite{chung1997spectral}, time series \cite{medvedev2012stochastic}, combinational problems \cite{medvedev2014nonlinear}, \emph{etc}. In general, GHE can be formulated as follows:
\begin{equation}
\begin{cases}
\frac{d\bX_t}{dt}&=-\bL \bX_t,\, \\
\bX_0&=\bX,
\end{cases}
\label{eqn:continuous-diffusion-sa}
\end{equation}
where $\bX_t\ (t\geq0)$ refers to the evolved input features at time $t$, and $\bL$ refers to the graph Laplacian matrix. Here, for the brevity of analysis, we take the symmetrically normalized graph Laplacian for the augmented adjacency $\widetilde\bA$ and overload the notation as $\bL=\widetilde{\mathbf{D}}^{-\frac{1}{2}} \left(\widetilde{\bD}-\widetilde{\mathbf{A}}\right) \widetilde{\mathbf{D}}^{-\frac{1}{2}}=\bI-\bS$. 

As GHE is a continuous dynamics, in practice we need to rely on numerical methods to solve it. We find that SGC can be seen as a coarse finite difference of GHE. Specifically, we apply the forward Euler method to Eq.~\eqref{eqn:continuous-diffusion-sa} with an interval $\Delta t$:
\begin{equation}
\begin{aligned}
\hat\bX_{t+\Delta t}=&\hat\bX_t - \Delta t {\bL}\hat\bX_t
=\hat\bX_t - \Delta t(\bI-\bS )\hat\bX_t 
=\left[(1-\Delta t)\bI+ \Delta t\bS \right]\hat\bX_t. 
\end{aligned}
\label{eqn:forward-euler}
\end{equation}
By involving the update rule for $K$ forward steps, we will get the final features $\hat\bX_T$ at the terminal time $T=K\cdot\Delta t$:
\begin{equation}
    \hat\bX_T=[\bS^{(\Delta t)} ]^K\bX, \text{ where } \bS^{(\Delta t)}=(1-\Delta t)\bI+ \Delta t\bS.
\end{equation}
Comparing to Eq.~\eqref{eqn:sgc-propagation}, we can see that the Euler discretization of GHE becomes SGC when the step size $\Delta t=1$. Specifically, the diffusion matrix $\bS^{(\Delta t)}$ reduces to the SGC diffusion matrix $\bS $ and the final node features, $\hat\bX_{T}$ and $\bX^{(K)}$,  become equivalent.
Therefore, SGC with $K$ propagation steps is essentially a finite difference approximation to GHE with $K$ forward steps (step size $\Delta t=1$ and terminal time $T=K$).

\subsection{Revealing the Fundamental Limitations of SGC}
\label{sec:limitations}

Based on the above analysis, we theoretically characterize several fundamental limitations of SGC: feature over-smoothing, large numerical errors and large learning risks. Proofs are in Appendix \ref{sec:appendix-proof}.

\begin{theorem}[Oversmoothing from a spectral view]
Assume that the eigendecomposition of the Laplacian matrix as $\bL =\sum_{i=1}^n\lambda_i\bu_i\bu_i^\top$, with eigenvalues $\lambda_i$ and eigenvectors $\bu_i$.
Then, the heat equation (Eq.~\eqref{eqn:continuous-diffusion-sa}) admits a closed-form solution at time $t$, known as the heat kernel $\bH_t=e^{-t\bL}=\sum_{i=1}^n e^{-\lambda_i t}\bu_i\bu_i^\top$. 
As $t\to\infty$, $\bH_t$ asymptotically converges to a non-informative equilibrium as $t\to\infty$, due to the non-trivial (i.e., positive) eigenvalues vanishing:
\begin{equation}
\lim_{t\to\infty}e^{-\lambda_it}=
\begin{cases}
0, &\text{ if } \lambda_i>0\\
1, &\text{ if } \lambda_i=0
\end{cases},\
i=1,\dots,n.
\end{equation}
\label{thm:aymptotic}
\end{theorem}

\begin{remark}\normalfont
In SGC, $T=K\cdot\Delta t=K$. Thus according to Theorem \ref{thm:aymptotic}, a large number of propagation steps $K\to\infty$ will inevitably lead to over-smoothed non-informative features.
\label{rmk:asymptotic}
\end{remark}

\begin{theorem}[Numerical errors]
\label{them2}
For the initial value problem in Eq.~\eqref{eqn:continuous-diffusion-sa} with finite terminal time $T$, the numerical error of the forward Euler method in Eq.~\eqref{eqn:forward-euler} with $K$ steps can be upper bounded by
\begin{equation}
\left\Vert \be^{(K)}_T\right\Vert\leq \frac{T\Vert\bL\Vert\Vert\bX_0\Vert}{2K}\left(e^{T\Vert\bL\Vert}-1\right).
\label{eqn:numerical-error}
\end{equation}
\label{thm:numeircal}
\end{theorem}

\begin{remark}\normalfont
Since $T=K$ in SGC, the upper bound reduces to $c\cdot\left(e^{T\Vert\bL\Vert}-1\right)$ ($c$ is a constant). We can see that the numerical error will increase exponentially with more propagation steps.
\label{rmk:numerical}
\end{remark}

\begin{theorem}[Learning risks] Consider a simple linear regression problem $(\bX,\bY)$ on graph, where the observed input features $\bX$ are generated by corrupting the ground truth features $\bX_c$ with the following inverse graph diffusion with time $T^*:$
\begin{equation}
\frac{d\widetilde{\bX}_t}{dt}=\bL\widetilde\bX_t, ~\text{ where }\widetilde\bX_0=\bX_c ~\text{ and }~ \widetilde\bX_{T^*}=\bX. 
\label{eqn:example-inverse-heat-equation}
\end{equation}
Denote the population risk with ground truth features as $R(\bW)=\bbE\left\Vert \bY-\bX_c\bW\right\Vert^2$ and that of Euler method applied input $\bX$ (Eq.~\eqref{eqn:forward-euler}) as $\hat R(\bW)=\bbE\left\Vert \bY-\left[\bS^{(\Delta t)}\right]^K\bX\bW\right\Vert^2.$
Supposing that $\bbE\Vert\bX_c\Vert^2=M<\infty$, we have the following upper bound:
\begin{align}
\hat R(\bW)< R(\bW)+2\Vert\bW\Vert^2\left(\bbE\left\Vert\be^{(K)}_{\hat T}\right\Vert^2+M\left\Vert e^{T^\star\bL}\right\Vert^2\cdot\left\Vert e^{-T^\star\bL}-e^{-\hat T\bL}\right\Vert^2\right).  
\label{eqn:learning-risk}
\end{align}
\label{thm:learning-risk}
\end{theorem}

\begin{remark}
\normalfont
Following Theorem \ref{thm:learning-risk}, we can see that the upper bound can be minimized by finding an optimal terminal time such that $\hat T=T^\star$ and minimizing the numerical error $\left\Vert\be^{(K)}_{\hat T}\right\Vert$. 
While SGC fixes the step size $\Delta t=1$, thus $T$ and $K$ are coupled together, which makes it less flexible to minimize the risk in Eq.~\eqref{eqn:learning-risk}.  
\label{rmk:risks}
\end{remark}

\section{The Proposed Decoupled Graph Convolution (DGC)}
\label{sec:method}
In this section, we introduce our proposed Decoupled Graph Convolution (DGC) and discuss how it overcomes the above limitations of SGC. 

\subsection{Formulation}

Based on the analysis in Section \ref{sec:limitations}, we need to resolve the coupling between propagation steps $K$ and terminal time $T$ caused by the fixed time interval $\Delta t=1$. Therefore, we regard the terminal time $T$ and the propagation steps $K$ as two \emph{free} hyper-parameters in the numerical integration via a flexible time interval. In this way, the two parameters can play different roles and cooperate together to attain better results: 1) we can flexibly choose $T$ to tradeoff between under-smoothing and over-smoothing to find a sweet spot for each dataset; and 2) given an optimal terminal time $T$, we can also flexibly increase the propagation steps $K$ for better numerical precision with $\Delta t=T/K\to0$. In practice, a moderate number of steps is sufficient to attain the best classification accuracy, hence we can also choose a minimal $K$ among the best for computation efficiency.

Formally, we propose our Decoupled Graph Convolution (DGC) as follows:
\begin{equation}
\hat{\mathbf{Y}}_{\mathrm{DGC}}=\operatorname{softmax}\left(\hat\bX_T\boldsymbol{\Theta}\right), \text{ where } \hat\bX_T=\operatorname{ode\_int}(\bX, \Delta t, K).
\end{equation}
Here $\operatorname{ode\_int}(\bX,\Delta t,K)$ refers to the numerical integration of the graph heat equation that starts from $\bX$ and runs for $K$ steps with step size $\Delta t$ . Here, we consider two numerical schemes: the forward Euler method and the Runge-Kutta (RK) method. 

\textbf{DGC-Euler}. As discussed previously, the forward Euler gives an update rule as in Eq.~\eqref{eqn:forward-euler}. With terminal time $T$ and step size $\Delta t=T/K$, we can obtain $\hat \bX_T$ after $K$ propagation steps:
\begin{equation}
    \hat\bX_T=\left[\bS^{(T/K)} \right]^K\bX, \text{ where }
    \bS^{(T/K)} =(1-T/K)\cdot\bI+ (T/K) \cdot\bS.
    \label{eqn:euler-feature-propagation}
\end{equation}

\textbf{DGC-RK.} Alternatively, we can apply higher-order finite difference methods to achieve better numerical precision, at the cost of more function evaluations at intermediate points. One classical method is the 4th-order Runge-Kutta (RK) method, which proceeds with
\begin{equation}
\hat\bX_{t+\Delta t}=\hat\bX_t+\frac{1}{6}\Delta t \left(\bR_{1}+2\bR_{2}+2 \bR_{3}+\bR_{4}\right)\overset{\Delta}{=}\bS^{(\Delta t)}_{\rm RK}\hat\bX_t,
\end{equation}
where 
\begin{equation}
\begin{aligned}
\bR_{1}=\hat\bX_k,~
\bR_{2}=\hat\bX_{k}-\frac{1}{2}\Delta t\bL\bR_1,~
\bR_{3}=\hat\bX_{k}-\frac{1}{2}\Delta t\bL\bR_2,~
\bR_{4}=\hat\bX_{k}-\Delta t\bL\bR_3.
\end{aligned}
\end{equation}
Replacing the propagation matrix $\bS^{(T/K)}$ in DGC-Euler with the RK-matrix $\bS^{(T/K)}_{\rm RK}$, we can get a 4th-order model, namely DGC-RK, whose numerical error can be  reduced to $O(1/K^4)$ order. 

\textbf{Remark.} 
In GCN \cite{kipf2016semi}, a self-loop $\bI$ is heuristically introduced in the adjacency matrix $\widetilde{\bA}=\bA+\bI$ to prevent numerical instability with more steps $K$. Here, we notice that the DGC-Euler diffusion matrix $\bS^{(\Delta t)}  =(1-\Delta t)\bI + \Delta t\bS$ naturally incorporates the self-loop $\bI$ into the diffusion process as a momentum term, where  $\Delta t$ flexibly tradeoffs information from the self-loop and the neighborhood. 
Therefore, in DGC, we can also remove the self-loop from $\widetilde\bA$ and increasing $K$ is still numerically stable with fixed $T$. 
We name the resulting model as DGC-sym with symmetrically normalized adjacency matrix $\bS_{\rm sym}=\bD^{-\frac{1}{2}}\bA\bD^{-\frac{1}{2}}$, which aligns with the canonical normalized graph Laplacian $\bL_{\rm sym}={\mathbf{D}}^{-\frac{1}{2}} \left({\bD}-{\mathbf{A}}\right) {\mathbf{D}}^{-\frac{1}{2}}=\bI-\bS_{\rm sym}$ in the spectral graph theory \cite{chung1997spectral}. Comparing the two Laplacians from a spectral perspective, $\bL=\bI-\bS$ has a smaller spectral range than $\bL_{\rm sym}$ \cite{wu2019simplifying}. According to Theorem \ref{thm:numeircal}, $\bL$ will have a faster convergence rate of numerical error.

\begin{figure}
    \centering
    \includegraphics[width=\textwidth]{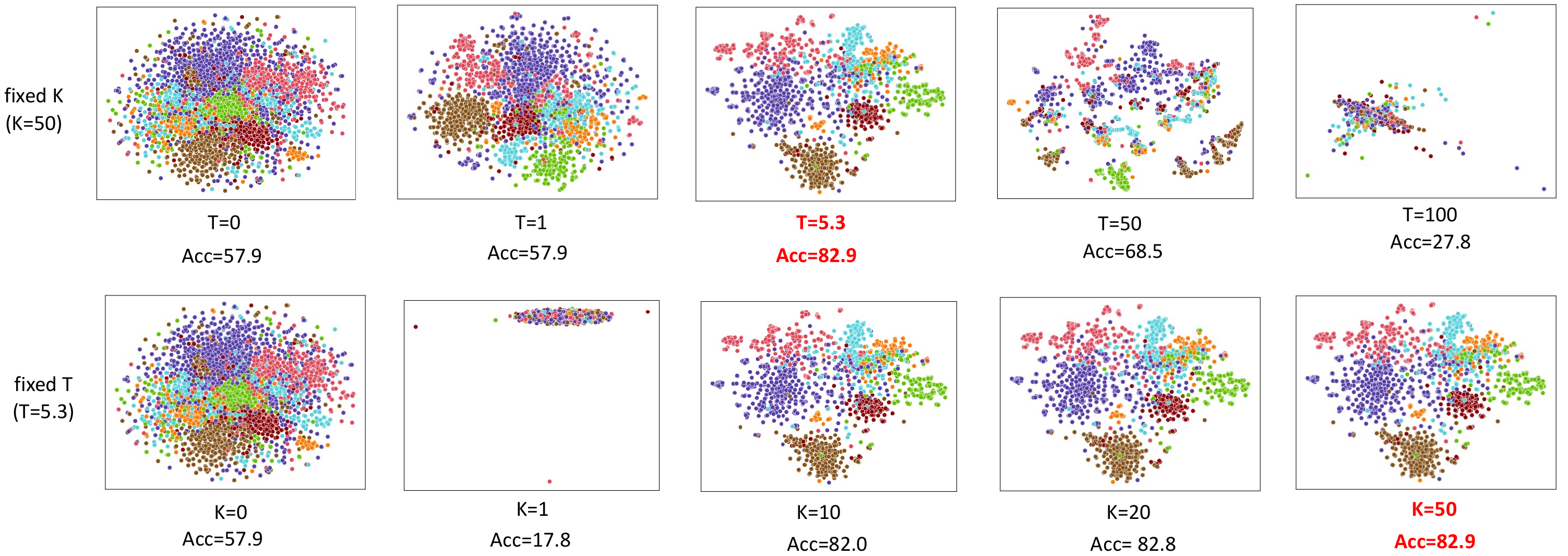}
    \caption{t-SNE input feature visualization and the corresponding test accuracy (\%) under different terminal time ($T$) and different number of propagation steps ($K$). Experiments are conducted with ours DGC-Euler model on the Cora dataset. Each point represents a node in the graph and its color denotes the class of the node.}
    \label{fig:t-SNE}
\end{figure}

\subsection{Verifying the Benefits of DGC}
\label{sec:understanding-DGC}

Here we demonstrate the advantages of DGC both theoretically and empirically.

\textbf{Theoretical benefits.} Revisiting Section \ref{sec:limitations}, DGC can easily alleviate the limitations of existing linear GCNs shown in Remarks \ref{rmk:asymptotic}, \ref{rmk:numerical}, \ref{rmk:risks}
by decoupling $T$ and $K$. 
\begin{itemize}
    \item For \textbf{Theorem \ref{thm:aymptotic}}, by choosing a fixed terminal time $T$ with optimal tradeoff, increasing the propagation steps $K$ in DGC will not lead to over-smoothing as in SGC;
    \item For \textbf{Theorem \ref{thm:numeircal}}, with $T$ is fixed, using more propagation steps ($K\to\infty$) in DGC  will help minimize the numerical error $\left\Vert\be^{(K)}_T\right\Vert$ with a smaller step size $\Delta t=T/K\to 0$;
    \item For \textbf{Theorem \ref{thm:learning-risk}}, by combining a flexibly chosen optimal terminal time $T^*$ and minimal numerical error with a large number of steps $K$, we can get minimal learning risks.
\end{itemize}

\textbf{Empirical evidence.} To further provide an intuitive understanding of DGC, we visualize the propagated input features of our proposed DGC-Euler on the Cora dataset in Figure \ref{fig:t-SNE}. The first row shows that there exists an optimal terminal time $T^*$ for each dataset with the best feature separability (\textit{e.g.}, 5.3 for Cora). Either a smaller $T$ (under-smooth) or a larger $T$ (over-smooth) will mix the features up and make them more indistinguishable, which eventually leads to lower accuracy. From the second row, we can see that, with fixed optimal $T$, too large step size $\Delta t$ (\ie too small propagation steps $K$) will lead to feature collapse, while gradually increasing the propagation steps $K$ makes the nodes of different classes more separable and improve the overall accuracy.

\begin{table}[t]
    \centering
    \caption{A comparison of propagation rules. Here $\bX^{(k)}\in\cX$ represents input features after $k$ feature propagation steps and $\bX^{(0)}=\bX$; $\bH^{(k)}$ denotes the hidden features of non-linear GCNs at layer $k$; $\bW$ denotes the weight matrix; $\sigma$ refers to a activation function; $\alpha,\beta$ are coefficients.}
    \begin{tabular}{lcc}
    \toprule
    Method & Type & Propagation rule \\
    \midrule
    
    GCN \cite{kipf2016semi}   & Non-linear & $\bH^{(k)}=\sigma\left(\bS\bH^{(k-1)}\bW^{(k-1)}\right)$ \\
    APPNP \cite{klicpera2018predict} & Non-linear & $\bH^{(k)}=(1-\alpha)\bS\bH^{(k-1)}+\alpha\bH^{(0)}$  \\
    CGNN \cite{xhonneux2020continuous}  & Non-linear & $\bH^{(k)}=(1-\alpha)\bS\bH^{(k-1)}\bW+\bH^{(0)}$  \\
    \midrule
    SGC \cite{wu2019simplifying}    & Linear & $\bX^{(k)}= \bS\bX^{(k-1)})$  \\
    
    \textbf{DGC-Euler} (ours) & Linear & $\bX^{(k)}= (1-T/K)\cdot\bX^{(k-1)} + (T/K)\cdot\bS\bX^{(k-1)}$  \\
    \bottomrule
    \end{tabular}
    \label{tab:propagation-rule-comparison}
\end{table}

\subsection{Discussions}

To highlight the difference of DGC to previous methods, we summarize their propagation rules in Table \ref{tab:propagation-rule-comparison}.
For non-linear methods, GCN \cite{kipf2016semi} uses the canonical propagation rule which has the oversmoothing issue, while APPNP \cite{klicpera2018predict} and CGNN \cite{xhonneux2020continuous} address it by further aggregating the initial hidden state $\bH^{(0)}$ repeatedly at each step. In particular, we emphasize that our DGC-Euler is different from APPNP in terms of the following aspects: 1) DGC-Euler is a linear model and propagates on the input features $\bX^{(k-1)}$, while APPNP is non-linear and propagates on non-linear embedding $\bH^{(k-1)}$; 2) at each step, APPNP aggregates features from the \emph{initial} step $\bH^{(0)}$, while DGC-Euler aggregates features from the \emph{last} step $\bX^{(k-1)}$; 3) APPNP aggregates a large amount $(1-\alpha)$ of the propagated features $\bS\bH^{(k-1)}$ while DGC-Euler only takes a small step $\Delta t$ ($T/K$) towards the new features $\bS\bX^{(k-1)}$. For linear methods, 
SGC has several fundamental limitations as analyzed in Section \ref{sec:limitations}, while DGC addresses them by flexible and fine-grained numerical integration of the propagation process.

Our dissection of linear GCNs also suggests a different understanding of the over-smoothing problem. As shown in Theorem \ref{thm:aymptotic}, over-smoothing is an inevitable phenomenon of (canonical) GCNs, while we can find a terminal time to achieve an optimal tradeoff between under-smoothing and over-smoothing. However, we cannot expect more layers can bring more profit if the terminal time goes to infinity, that is, the benefits of more layers can only be obtained under a proper terminal time.

\section{Experiments}
\label{sec:experiments}
In this section, we conduct a comprehensive  analysis on DGC and compare it against both linear and non-linear GCN variants on a collection of benchmark datasets.

\begin{table}[t]
\centering
\caption{Test accuracy (\%) of semi-supervised node classification on citation networks.}
\begin{tabular}{llccc}
\bottomrule
\rule{0pt}{1.02\normalbaselineskip} 
Type & Method & Cora & Citeseer  & Pubmed\\
\midrule\multirow{9}{*}{Non-linear} 
& GCN \cite{kipf2016semi} & 81.5 & 70.3 & 79.0 \\
& GAT \cite{velivckovic2018graph}  & 83.0 $\pm$ 0.7 & 72.5 $\pm$ 0.7 & 79.0 $\pm$ 0.3 \\
& GraphSAGE \cite{hamilton2017inductive} & 82.2 & 71.4 & 75.8  \\
& JKNet \cite{xu2018representation} & 81.1 & 69.8 & 78.1 \\
& APPNP \cite{klicpera2018predict} & 83.3 & 71.8 & 80.1 \\
& GWWN \cite{xu2019graph} & 82.8 & 71.7 & 79.1 \\
& GraphHeat \cite{xu2020graph} & 83.7 & 72.5 & 80.5 \\
&CGNN \cite{xhonneux2020continuous} & 84.2 $\pm$ 0.6 & 71.8 $\pm$ 0.7 & 76.8 $\pm$ 0.6 \\
&GCDE \cite{poli2019graph} & 83.8 $\pm$ 0.5 & 72.5 $\pm$ 0.5 & 79.9 $\pm$ 0.3 \\
\midrule
\multirow{5}{*}{Linear} &  Label Propagation \cite{zhu2003semi} & 45.3 &  68.0 & 63.0 \\
& DeepWalk \cite{perozzi2014deepwalk} & 70.7 $\pm$ 0.6 & 51.4 $\pm$ 0.5 & 76.8 $\pm$ 0.6 \\
& SGC \cite{wu2019simplifying} & 81.0 $\pm$ 0.0 & 71.9 $\pm$ 0.1 & 78.9 $\pm$ 0.0\\
& SGC-PairNorm \cite{zhao2019pairnorm} & 81.1 & 70.6 & 78.2 \\
& {SIGN-linear \cite{rossi2020sign}} & 81.7 & 72.4 & 78.6 \\
\cline{2-5} 
\rule{0pt}{1.02\normalbaselineskip} 
& {\textbf{DGC} (ours)} & \textbf{83.3 $\pm$ 0.0} & \textbf{73.3 $\pm$ 0.1} & \textbf{80.3 $\pm$ 0.1} \\
\bottomrule
\end{tabular}
\label{tab:citation-results}
\end{table}

\subsection{Performance on Semi-supervised Node Classification}

\textbf{Setup.} 
For semi-supervised node classification, we use three standard citation networks, Cora, Citeseer, and Pubmed \cite{sen2008collective} and adopt the standard data split as in  \cite{kipf2016semi,velivckovic2018graph,xu2019graph,xu2020graph,poli2019graph}. Here we compare our DGC against several representative non-linear and linear methods that also adopts the standard data split. For non-linear GCNs, we include 1) classical baselines like GCN \cite{kipf2016semi}, GAT \cite{velickovic2019deep}, GraphSAGE \cite{hamilton2017inductive}, APPNP \cite{klicpera2018predict} and JKNet \cite{xu2018representation}; 2) spectral methods using graph heat kernel \cite{xu2019graph,xu2020graph}; and 3) continuous GCNs \cite{poli2019graph,xhonneux2020continuous}. For linear methods, we present the results of Label Propagation \cite{zhu2003semi}, DeepWalk \cite{perozzi2014deepwalk}, SGC (linear GCN) \cite{wu2019simplifying} as well as its regularized version SGC-PairNorm \cite{zhao2019pairnorm}. We also consider a linear version of SIGN \cite{rossi2020sign}, SIGN-linear, which extends SGC by aggregating features from multiple propagation stages ($K=1,2,\dots$). For DGC, we adopt the Euler scheme, \ie DGC-Euler (Eq.~\eqref{eqn:euler-feature-propagation}) by default for simplicity. We report results averaged over 10 random runs.  Data statistics and training details are in Appendix \ref{sec:appendix-training-configurations}.

We compare DGC against both linear and non-linear baselines for the semi-supervised node classification task, and the results are shown in Table \ref{tab:citation-results}.

\textbf{DGC v.s. linear methods.} 
We can see that DGC shows significant improvement over previous linear methods across three datasets. 
In particular, compared to SGC (previous SOTA methods), DGC obtains 83.3 \vs 81.0 on Cora, 73.3 \vs 71.9 on Citeseer and 80.3 \vs 78.9 on Pubmed. This shows that in real-world datasets, a flexible and fine-grained integration by decoupling $T$ and $K$ indeed helps improve the classification accuracy of SGC by a large margin. 
{
Besides, DGC also outperforms the multi-scale SGC, SIGN-linear, suggesting that multi-scale techniques cannot fully solve the limitations of SGC, while DGC can overcome these limitations by decoupling $T$ and $K$. As discussed in Appendix \ref{sec:appendix-sign}, DGC still shows clear advantages over SIGN when controlling the terminal time $T$ while being more computationally efficient, which indicates that the advantage of DGC is not only a real-valued $T$, but also the improved numerical precision by adopting a large $K$.}

\textbf{DGC v.s. non-linear models.} Table \ref{tab:citation-results} further shows that DGC, as a linear model, even outperforms many non-linear GCNs on semi-supervised tasks. First, DGC improves over classical GCNs like GCN \cite{kipf2016semi}, GAT \cite{velivckovic2018graph} and GraphSAGE \cite{hamilton2017inductive} by a large margin. 
Also, DGC is comparable to, and sometimes outperforms, many modern non-linear GCNs. For example, DGC shows a clear advantage over multi-scale methods like JKNet \cite{xu2018representation} and APPNP \cite{klicpera2018predict}. 
DGC is also comparable to spectral methods based on graph heat kernel, \eg GWWN \cite{xu2019graph}, GraphHeat \cite{xu2020graph}, while being much more efficient as a simple linear model. Besides, compared to non-linear continuous models like GCDE \cite{poli2019graph} and CGNN \cite{xhonneux2020continuous}, DGC also achieves comparable accuracy only using a simple linear dynamic.

\begin{table}[t]
\centering
\caption{Test accuracy (\%) of fully-supervised node classification on citation networks.}
\begin{tabular}{llccc}
\bottomrule
\rule{0pt}{1.02\normalbaselineskip} 
Type & Method & Cora & Citeseer  & Pubmed\\
\midrule\multirow{3}{*}{Non-linear} 
& GCN \cite{kipf2016semi} & 85.8 & 73.6 & 88.1 \\
& GAT \cite{velivckovic2018graph}  & 86.4 & 74.3 & 87.6 \\
& JK-MaxPool \cite{xu2018representation} & 89.6 & 77.7 & - \\
& JK-Concat \cite{xu2018representation} & 89.1 & 78.3 & - \\
& JK-LSTM \cite{xu2018representation} & 85.8 & 74.7 & - \\
& APPNP \cite{klicpera2018predict} & 90.2 & 79.8 & 86.3 \\
\midrule
\multirow{2}{*}{Linear} 
& SGC \cite{wu2019simplifying} & 85.8 & 78.1 & 83.3 \\
\cline{2-5} 
\rule{0pt}{1.02\normalbaselineskip} 
& {\textbf{DGC} (ours)} & \textbf{88.2 $\pm$ 0.0} & \textbf{78.7 $\pm$ 0.0} & \textbf{89.4 $\pm$ 0.0} \\
\bottomrule
\end{tabular}
\label{tab:citation-random}
\end{table}

\subsection{Performance on Fully-supervised Node Classification}

\textbf{Setup.} For fully-supervised node classification, we also use the three citation networks, Cora, Citeseer and Pubmed, but instead randomly split the nodes in three citation networks into 60\%, 20\% and 20\% for training, validation and testing, following the previous practice in \cite{xu2018representation}. 
Here, we include the baselines that also have reported results on the fully supervised setting, such as GCN \cite{kipf2016semi}, GAT \cite{velivckovic2018graph} (reported baselines in \cite{xu2018representation}), and the three variants of JK-Net: JK-MaxPool, JK-Concat and JK-LSTM \cite{xu2018representation}. Besides, we also reproduce the result of APPNP \cite{klicpera2018predict} for a fair comparison.  Dataset statistics and training details are described in Appendix. 

\textbf{Results.}
The results of the fully-supervised semi-classification task are basically consistent with the semi-supervised setting. As a linear method, DGC not only improves the state-of-the-art linear GCNs by a large margin, but also outperforms GCN \cite{kipf2016semi}, GAT \cite{velivckovic2018graph} significantly. Beside, DGC is also comparable to multi-scale methods like JKNet \cite{xu2018representation} and APPNP \cite{klicpera2018predict}, showing that a good linear model like DGC is also competitive for  fully-supervised tasks. 

\subsection{Performance on Large Scale Datasets}

\begin{wraptable}{r}{.5\textwidth}
\centering
\caption{Test accuracy (\%) comparison with inductive methods on  on  a large scale dataset, Reddit. Reported results are averaged over 10 runs. OOM: out of memory.}
\begin{tabular}{llc}
\bottomrule
\rule{0pt}{1.02\normalbaselineskip} 
Type & Method & Acc. \\
\midrule\multirow{6}{*}{Non-linear} 
& GCN \cite{kipf2016semi} & {OOM} \\
& FastGCN \cite{chen2018fastgcn} & 93.7 \\
& GraphSAGE-GCN \cite{hamilton2017inductive} & 93.0 \\
& GraphSAGE-mean \cite{hamilton2017inductive} & 95.0 \\
& GraphSAGE-LSTM \cite{hamilton2017inductive} & 95.4 \\
& APPNP \cite{klicpera2018predict} & 95.0 \\
\midrule \multirow{3}{*}{Linear} 
& RandDGI \cite{velickovic2019deep} & 93.3 \\
& SGC \cite{wu2019simplifying} & 94.9 \\
\cline{2-3}
\rule{0pt}{1.02\normalbaselineskip} 
& \textbf{DGC} (ours)  & \textbf{95.8} \\
\bottomrule
\end{tabular}
\label{tab:reddit-results}
\end{wraptable}

\textbf{Setup.} 
More rigorously, we also conduct the comparison on a large scale node classification dataset, the Reddit networks \cite{hamilton2017inductive}. Following SGC \cite{wu2019simplifying}, we adopt the inductive setting, where we use the subgraph of training nodes as training data and use the whole graph for the validation/testing data. For a fair comparison, we use the same training configurations as SGC \cite{wu2019simplifying} and include its reported baselines, such as GCN \cite{kipf2016semi}, FastGCN \cite{chen2018fastgcn}, three variants of GraphSAGE \cite{hamilton2017inductive}, and RandDGI (DGI with randomly initialized encoder) \cite{velickovic2019deep}.
We also include APPNP \cite{klicpera2018predict} for a comprehensive comparison.

\textbf{Results.} 
We can see DGC still achieves the best accuracy among linear methods and improve 0.9\% accuracy over SGC. Meanwhile, it is superior to the three variants of GraphSAGE as well as APPNP. Thus, DGC is still the \sota linear GCNs and competitive against nonlinear GCNs on large scale datasets.

\begin{figure}[t]
    \centering
    \includegraphics[width=\linewidth]{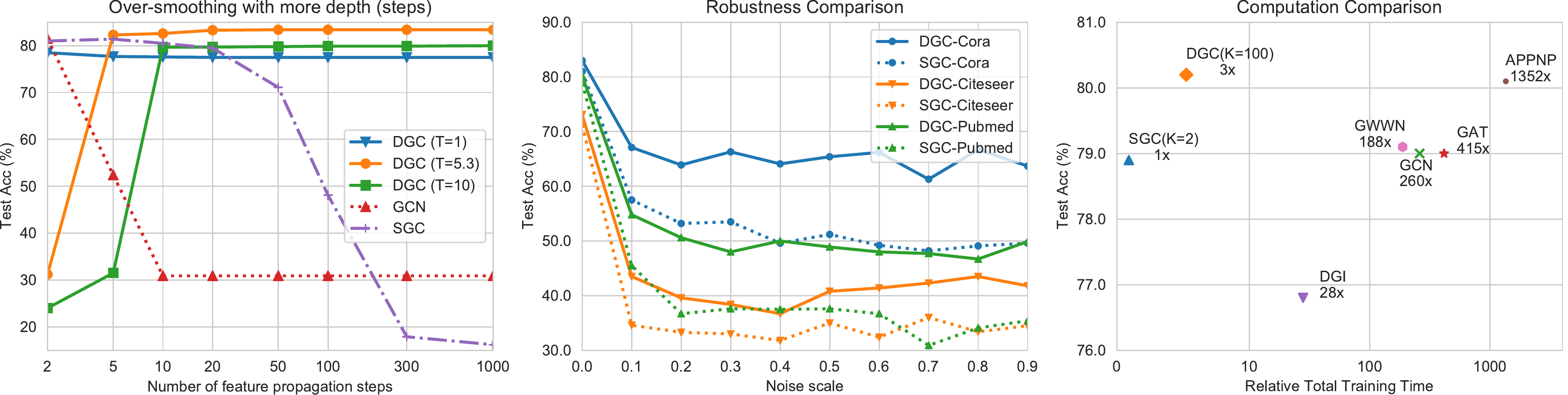}
    \caption{ Left: test accuracy (\%) with increasing feature propagation steps on Cora. Middle: comparison of robustness under different noise scales $\sigma$ on three citation networks. Right: a comparison of relative total training time for 100 epochs on the Pubmed dataset.}
    \label{fig:depth-comparison}
\end{figure}

\subsection{Empirical Understandings of DGC}

\textbf{Setup.} Here we further provide a comprehensive analysis of DGC.
First, we compare its over-smoothing behavior and computation time against previous methods. Then we analyze several factors that affect the performance of DGC, including the Laplacian matrix $\bL$, the numerical schemes and the terminal time $T$. Experiments are conducted on the semi-supervised learning tasks and we adopt DGC-Euler with the default hyper-parameters unless specified.

{\textbf{Non-over-smoothing with increasing steps.}} In the left plot of Figure \ref{fig:depth-comparison}, we compare different GCNs with increasing model depth (non-linear GCNs) or propagation steps (linear GCNs) from 2 to 1000. Baselines include SGC \cite{wu2019simplifying}, GCN \cite{kipf2016semi}, and our DGC with three different terminal time $T$ (1, 5.3, 10). First we notice that SGC and GCN fail catastrophically when increasing the depth, which is consistent with the previously observed over-smoothing phenomenon. 
Instead, all three DGC variants can benefit from increased steps. Nevertheless, the final performance will degrade if the terminal time is either too small ($T=1$, under-smoothing) or too large ($T=10$, over-smoothing). DGC enables us to flexibly find the optimal terminal time ($T=5.3$). Thus, we can obtain the optimal accuracy with an optimal tradeoff between under-smoothing and over-smoothing.

\textbf{Robustness to feature noise.} In real-world applications, there are plenty of noise in the collected node attributes, thus it is crucial for GCNs to be robust to input noise \cite{bojchevski2019certifiable}. Therefore, we compare the robustness of SGC and DGC against Gaussian noise added to the input features, where $\sigma$ stands for the standard deviation of the noise. Figure \ref{fig:depth-comparison} (middle) shows that DGC is significantly more robust than SGC across three citation networks, and the advantage is clearer on larger noise scales. {As discussed in Theorem \ref{thm:learning-risk}, the diffusion process in DGC can be seen as a denoising procedure, and consequently, DGC's robustness to feature noise can be contributed to the optimal tradeoff between over-smoothing and under-smoothing with a flexible choice of $T$ and $K$. In comparison, SGC is not as good as DGC because it cannot find such a sweet spot accurately.}

\textbf{Computation time.} In practice, linear GCNs can accelerate training by pre-processing features with all propagation steps and storing them for the later model training. Since pre-processing costs much fewer time than training ($<$5\% in SGC), linear GCNs could be much faster than non-linear ones. As shown in Figure \ref{fig:depth-comparison} (right), DGC is slightly slower (3$\times$) than SGC, but DGC achieves much higher accuracy. Even so, DGC is still much faster than non-linear GCNs ($>$100$\times$). {Indeed, as further shown in Table \ref{tab:computation-time}, the computation overhead of DGC over SGC mainly lies in the preprocessing stage, which is very small in SGC and only leads to around twice longer total time. Instead, GCN is much slower as it involves propagation in each training loop, leading to much slower training.}

\begin{table}[t]
\centering
\caption{Comparison of explicit computation time of different training stages on the Pubmed dataset with a single NVIDIA GeForce RTX 3090 GPU.}
\begin{tabular}{lllll}
\toprule
Type & Method      & Preprocessing Time & Training Time & Total Time \\ \midrule
\multirow{3}{*}{Linear} & SGC ($K=2$) \cite{wu2019simplifying}   & 3.8 ms             & 61.5 ms       & 65.3 ms    \\
& DGC ($K=2$)  (ours)  & 3.8 ms             & 61.5 ms       & 65.3 ms    \\
& DGC ($K=100$) (ours) & 169.2 ms           & 55.8 ms       & 225.0 ms   \\
\midrule
Nonlinear & GCN  \cite{kipf2016semi}        & 0                  & 17.0 s        & 17.0 s    \\
\bottomrule
\end{tabular}
\label{tab:computation-time}
\end{table}

\textbf{Graph Laplacian.} 
As shown in Figure \ref{fig:model-analysis} (left), in DGC, both the two  Laplacians, $\bL$ (with self-loop) and $\bL_{\rm sym}$ (without self-loop), can consistently benefit from more propagation steps without leading to numerical issues. 
Further comparing the two Laplacians, we can see that the augmented Laplacian $\bL$ obtains higher test accuracy than the canonical Laplacian $\bL_{\rm sym}$ and requires fewer propagation steps $K$ to obtain good results, which  
could also be understood from our analysis in Section \ref{sec:limitations}.

\begin{figure}[!t]
    \centering
    \includegraphics[width=\textwidth]{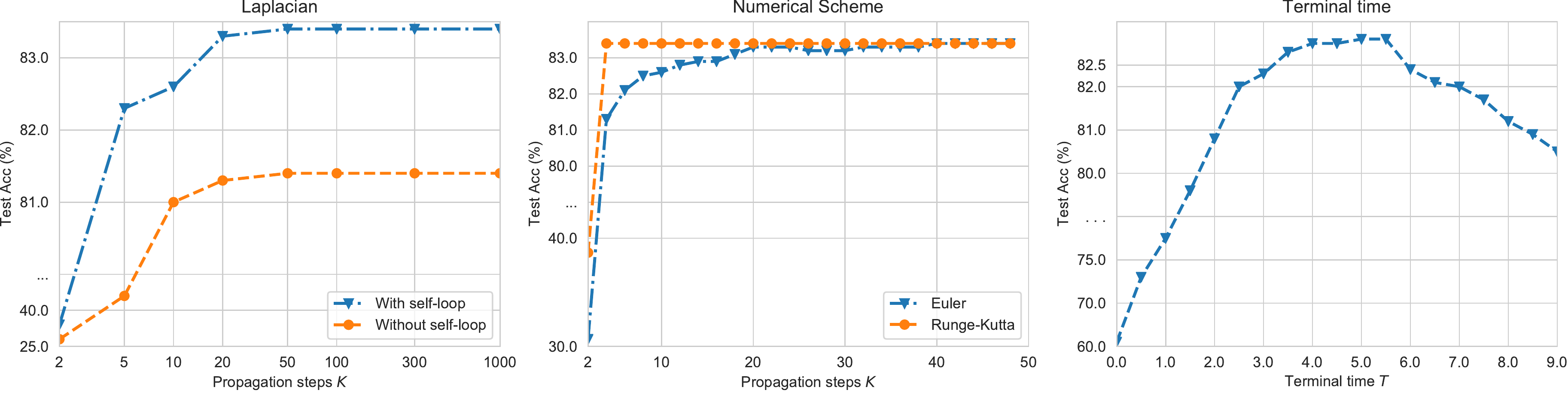}
    \caption{Algorithmic analysis of our proposed DGC. Left: test accuracy (\%) of two kinds of Laplacian, $\bL=\bI-\bS$ (with self-loop) and $\bL_{\rm sym}=\bI-\bS_{\rm sym}$ (without self-loop), with increasing steps $K$ and fixed time $T$ on Cora. Middle: test accuracy (\%) of two numerical schemes, Euler and Runge-Kutta, with increasing steps $K$ and fixed $T$ under fixed terminal time on Cora. Right: test accuracy (\%) with varying terminal time $T$ and fixed steps $K$ on Cora.
    }
    \label{fig:model-analysis}
    \vspace{-0.15in}
\end{figure}

\textbf{Numerical scheme.} 
By comparing different numerical schemes in Figure \ref{fig:model-analysis} (middle), we find that the Runge-Kutta method demonstrates better accuracy than the Euler method with a small $K$. Nevertheless, as $K$ increases, the difference gradually vanishes. 
Thus, the Euler method is sufficient for DGC to achieve good performance and it is more desirable in term of its simplicity and efficiency.

\begin{wraptable}{r}{.5\textwidth}
    \centering
    \caption{Optimal terminal time $T^*$ on the transductive task, Pubmed, and the inductive task, Reddit, with different Laplacians. }
    \begin{tabular}{lccc}
    \toprule
    Dataset & Laplacian & $T^*$ & Acc \\
    \midrule
    \multirow{2}{*}{Pubmed} & $\bI-\bS $ & 6.0& \textbf{80.3} \\
     & $\bI-\bS_{\rm sym}$ & 6.0  & 79.8 \\
    \midrule
    \multirow{2}{*}{Reddit} & $\bI-\bS $ & 2.7  & 95.5 \\
     & $\bI-\bS_{\rm sym}$ & 2.6 & \textbf{95.8} \\
    \bottomrule
    \end{tabular}
    \label{tab:list-T-K-acc}
    \end{wraptable}

\textbf{Terminal time $T$.} In Figure \ref{fig:model-analysis} (right), we compare the test accuracy with different terminal time $T$. We show that indeed, in real-world datasets, there exists a sweet spot that achieves the optimal tradeoff between under-smoothing and over-smoothing. 
In Table \ref{tab:list-T-K-acc}, we list the best terminal time that we find on two large graph datasets, Pubmed and Reddit. We can see that $T$ is almost consistent across different Laplacians on each dataset, which suggests that the optimal terminal time $T^*$ is an intrinsic property of the dataset.

\section{Conclusions}
In this paper, we have proposed Decoupled Graph Convolution (DGC), which improves significantly over previous linear GCNs through decoupling the terminal time and feature propagation steps from a continuous perspective. Experiments show that our DGC is competitive with many modern variants of non-linear GCNs while being much more computationally efficient with much fewer parameters to learn.

Our findings suggest that, unfortunately, current GCN variants still have not shown significant advantages over a properly designed linear GCN.
We believe that this would attract the attention of the community to reconsider the actual representation ability of current nonlinear GCNs and propose new alternatives that can truly benefit from nonlinear architectures. 

\section*{Acknowledgement}
Yisen Wang is partially supported by the National Natural Science Foundation of China under Grant 62006153, and Project 2020BD006 supported by PKU-Baidu Fund. 
Jiansheng Yang is supported by the National Science Foundation of China under Grant No. 11961141007.
Zhouchen Lin is supported by the NSF China under Grants 61625301 and 61731018, Project 2020BD006 supported by PKU-Baidu Fund, and Zhejiang Lab (grant no. 2019KB0AB02).

\nocite{langley00}

\bibliography{main}

\begin{thebibliography}{10}

\bibitem{bergstra2015hyperopt}
James Bergstra, Brent Komer, Chris Eliasmith, Dan Yamins, and David~D Cox.
\newblock Hyperopt: a python library for model selection and hyperparameter
  optimization.
\newblock {\em Computational Science \& Discovery}, 8(1):014008, 2015.

\bibitem{bojchevski2019certifiable}
Aleksandar Bojchevski and Stephan G{\"u}nnemann.
\newblock Certifiable robustness to graph perturbations.
\newblock {\em NeurIPS}, 2019.

\bibitem{chen2018fastgcn}
Jie Chen, Tengfei Ma, and Cao Xiao.
\newblock {FastGCN}: fast learning with graph convolutional networks via
  importance sampling.
\newblock {\em ICLR}, 2018.

\bibitem{chen2018neural}
Ricky~TQ Chen, Yulia Rubanova, Jesse Bettencourt, and David~K Duvenaud.
\newblock Neural ordinary differential equations.
\newblock {\em NeurIPS}, 2018.

\bibitem{chung1997spectral}
Fan~RK Chung and Fan~Chung Graham.
\newblock {\em Spectral graph theory}.
\newblock American Mathematical Society, 1997.

\bibitem{hamilton2017inductive}
Will Hamilton, Zhitao Ying, and Jure Leskovec.
\newblock Inductive representation learning on large graphs.
\newblock {\em NeurIPS}, 2017.

\bibitem{he2016deep}
Kaiming He, Xiangyu Zhang, Shaoqing Ren, and Jian Sun.
\newblock Deep residual learning for image recognition.
\newblock {\em CVPR}, 2016.

\bibitem{kingma2014adam}
Diederik~P Kingma and Jimmy Ba.
\newblock Adam: A method for stochastic optimization.
\newblock {\em ICLR}, 2015.

\bibitem{kipf2016semi}
Thomas~N Kipf and Max Welling.
\newblock Semi-supervised classification with graph convolutional networks.
\newblock {\em ICLR}, 2017.

\bibitem{klicpera2018predict}
Johannes Klicpera, Aleksandar Bojchevski, and Stephan G{\"u}nnemann.
\newblock Predict then propagate: Graph neural networks meet personalized
  pagerank.
\newblock {\em ICLR}, 2019.

\bibitem{li2019deepgcns}
Guohao Li, Matthias Muller, Ali Thabet, and Bernard Ghanem.
\newblock {DeepGCNs}: Can gcns go as deep as cnns?
\newblock {\em CVPR}, 2019.

\bibitem{li2018deeper}
Qimai Li, Zhichao Han, and Xiao-Ming Wu.
\newblock Deeper insights into graph convolutional networks for semi-supervised
  learning.
\newblock {\em AAAI}, 2018.

\bibitem{liu1989limited}
Dong~C Liu and Jorge Nocedal.
\newblock On the limited memory {BFGS} method for large scale optimization.
\newblock {\em Mathematical Programming}, 45(1):503--528, 1989.

\bibitem{lu18d}
Yiping Lu, Aoxiao Zhong, Quanzheng Li, and Bin Dong.
\newblock Beyond finite layer neural networks: Bridging deep architectures and
  numerical differential equations.
\newblock {\em ICML}, 2018.

\bibitem{medvedev2012stochastic}
Georgi~S Medvedev.
\newblock Stochastic stability of continuous time consensus protocols.
\newblock {\em SIAM Journal on Control and Optimization}, 50(4):1859--1885,
  2012.

\bibitem{medvedev2014nonlinear}
Georgi~S Medvedev.
\newblock The nonlinear heat equation on dense graphs and graph limits.
\newblock {\em SIAM Journal on Mathematical Analysis}, 46(4):2743--2766, 2014.

\bibitem{perozzi2014deepwalk}
Bryan Perozzi, Rami Al-Rfou, and Steven Skiena.
\newblock {DeepWalk}: Online learning of social representations.
\newblock {\em SIGKDD}, 2014.

\bibitem{poli2019graph}
Michael Poli, Stefano Massaroli, Junyoung Park, Atsushi Yamashita, Hajime
  Asama, and Jinkyoo Park.
\newblock Graph neural ordinary differential equations.
\newblock {\em arXiv preprint arXiv:1911.07532}, 2019.

\bibitem{rong2019dropedge}
Yu~Rong, Wenbing Huang, Tingyang Xu, and Junzhou Huang.
\newblock {DropEdge}: Towards deep graph convolutional networks on node
  classification.
\newblock {\em ICLR}, 2019.

\bibitem{rossi2020sign}
Emanuele Rossi, Fabrizio Frasca, Ben Chamberlain, Davide Eynard, Michael
  Bronstein, and Federico Monti.
\newblock {SIGN}: Scalable inception graph neural networks.
\newblock {\em arXiv preprint arXiv:2004.11198}, 2020.

\bibitem{sen2008collective}
Prithviraj Sen, Galileo Namata, Mustafa Bilgic, Lise Getoor, Brian Galligher,
  and Tina Eliassi-Rad.
\newblock Collective classification in network data.
\newblock {\em AI Magazine}, 29(3):93--93, 2008.

\bibitem{velivckovic2018graph}
Petar Veli{\v{c}}kovi{\'c}, Guillem Cucurull, Arantxa Casanova, Adriana Romero,
  Pietro Lio, and Yoshua Bengio.
\newblock Graph attention networks.
\newblock {\em ICLR}, 2018.

\bibitem{velickovic2019deep}
Petar Velickovic, William Fedus, William~L Hamilton, Pietro Li{\`o}, Yoshua
  Bengio, and R~Devon Hjelm.
\newblock Deep graph infomax.
\newblock {\em ICLR}, 2019.

\bibitem{wu2019simplifying}
Felix Wu, Tianyi Zhang, Amauri Holanda~de Souza~Jr, Christopher Fifty, Tao Yu,
  and Kilian~Q Weinberger.
\newblock Simplifying graph convolutional networks.
\newblock {\em ICML}, 2019.

\bibitem{xhonneux2020continuous}
Louis-Pascal Xhonneux, Meng Qu, and Jian Tang.
\newblock Continuous graph neural networks.
\newblock {\em ICML}, 2020.

\bibitem{xu2020graph}
Bingbing Xu, Huawei Shen, Qi~Cao, Keting Cen, and Xueqi Cheng.
\newblock Graph convolutional networks using heat kernel for semi-supervised
  learning.
\newblock {\em arXiv preprint arXiv:2007.16002}, 2020.

\bibitem{xu2019graph}
Bingbing Xu, Huawei Shen, Qi~Cao, Yunqi Qiu, and Xueqi Cheng.
\newblock Graph wavelet neural network.
\newblock {\em ICLR}, 2019.

\bibitem{xu2018representation}
Keyulu Xu, Chengtao Li, Yonglong Tian, Tomohiro Sonobe, Ken-ichi Kawarabayashi,
  and Stefanie Jegelka.
\newblock Representation learning on graphs with jumping knowledge networks.
\newblock {\em ICML}, 2018.

\bibitem{yang2016revisiting}
Zhilin Yang, William Cohen, and Ruslan Salakhudinov.
\newblock Revisiting semi-supervised learning with graph embeddings.
\newblock {\em ICML}, 2016.

\bibitem{zhao2019pairnorm}
Lingxiao Zhao and Leman Akoglu.
\newblock {PairNorm}: Tackling oversmoothing in gnns.
\newblock {\em ICLR}, 2020.

\bibitem{zhu2003semi}
Xiaojin Zhu, Zoubin Ghahramani, and John~D Lafferty.
\newblock Semi-supervised learning using gaussian fields and harmonic
  functions.
\newblock {\em ICML}, 2003.

\end{thebibliography}
\bibliographystyle{plain}

\appendix

\section{Training Configurations}
\label{sec:appendix-training-configurations}

\textbf{Data statistics.} 
We summarize the data statistics in our experiments in Table \ref{tab:data-statistics}.

\begin{table}[h]
    \centering
    \caption{Dataset statistics of the three learning tasks in our experiments.}
    \begin{tabular}{llcccc}
    \toprule
    Learning Task & Dataset & Nodes & Edges & Train/Dev/Test Nodes & Split Ratio (\%) \\
    \midrule
\multirow{3}{*}{Semi-supervised} & Cora & 2,708 & 5,429 & 140/500/1,000 & 5.2/18.5/36.9 \\
& Citeseer & 3,327 & 4,732 & 120/500/1,000  & 3.6/15.0/30.1 \\
& Pubmed & 19,717 & 44,338 & 60/500/1,000 & 0.3/2.5/5.1 \\
\midrule
\multirow{3}{*}{Fully-supervised} & Cora & 2,708 & 5,429 & 1624/541/543 & 60.0/20.0/20.0\\
& Citeseer & 3,327 & 4,732 & 1996/665/666  & 60.0/20.0/20.0\\
& Pubmed & 19,717 & 44,338 & 11830/3943/3944  & 60.0/20.0/20.0\\
\midrule
Inductive (large-scale) & Reddit & 233K & 11.6M & 152K/24K/55K & 65.2/10.3/23.6 \\
\bottomrule
    \end{tabular}
    \label{tab:data-statistics}
\end{table}

\textbf{Training hyper-parameters.} For both fully and semi-supervised node classification tasks on the citation networks, Cora, Citeseer and Pubmed, we train our DGC following the hyper-parameters in SGC \cite{wu2019simplifying}. Specifically, we train DGC for 100 epochs using Adam \cite{kingma2014adam} with learning rate 0.2. For weight decay, as in SGC, we tune this hyperparameter on each dataset using hyperopt \cite{bergstra2015hyperopt} for 10,000 trails. For the large-scale inductive learning task on the Reddit network, we also follow the protocols of SGC \cite{wu2019simplifying}, where we use L-BFGS \cite{liu1989limited} optimizer for 2 epochs with no weight decay.

\section{Omitted Proofs}
\label{sec:appendix-proof}
\subsection{Proof of Theorem 1}

\begin{theorem}
The heat kernel $\bH_t=e^{-t\bL}$ admits the following eigen-decomposition,
\begin{align*}
\bH_t
&=\bU 
\begin{pmatrix}
e^{-\lambda_1t} & 0 & \cdots & 0\\
0 & e^{-\lambda_2t} & \cdots & 0\\
\vdots & \vdots &\ddots & \vdots\\
0 & 0 &\cdots & e^{-\lambda_nt}
\end{pmatrix}
\bU^\top. 
\end{align*}
As a result, with $\lambda_i\geq0$, we have 
\begin{equation}
\lim_{t\to\infty}e^{-\lambda_it}=
\begin{cases}
0, &\text{ if } \lambda_i>0\\
1, &\text{ if } \lambda_i=0
\end{cases},\
i=1,\dots,n.
\end{equation}
\end{theorem}

\begin{proof}
With the eigen-decomposition of the Laplacian $\bL=\bU\bLambda\bU^\top$, the heat kernel can be written equivalently as
\begin{align}
\bH_t&=e^{-t\bL}
=\sum_{k=0}^{\infty}\frac{t^k}{k!}(-\bL)^k=\sum_{k=0}^{\infty}\frac{t^k}{k!}\left[\bU(-\bLambda)\bU^\top\right]^k=\bU\left[\sum_{k=0}^{\infty}\frac{t^k}{k!}(-\bLambda)^k\right]\bU^T=\bU e^{-t\bLambda}\bU^T,
\end{align}
which corresponds to the eigen-decomposition of the heat kernel with eigen-vectors in the orthogonal matrix $\bU$ and eigven-values in the diagonal matrix $e^{-t\bLambda}$. Now it is easy to see the limit behavior of the heat kernel as $t\to\infty$ from the spectral domain.
\end{proof}

\subsection{Proof of Theorem 2}
\begin{theorem}
\label{them2}
For the general initial value problem 
\begin{equation}
\begin{cases}
\frac{d\bX_t}{dt}&=-\bL\bX_t,\, \\
\bX_0&=\bX,
\end{cases}
\label{eqn:general-heat-equation}
\end{equation}
with any finite terminal time $T$, the numerical error of the forward Euler method 
\begin{equation}
    \hat\bX^{(K)}_T=\left(\bI-\frac{T}{K}\bL\right)^K\bX_0.
    \label{eqn:forward-euler}
\end{equation}
with $K$ propagation steps can be upper bounded by
\begin{equation}
\Vert \be^{(K)}_T\Vert\leq \frac{T\Vert\bL\Vert\Vert\bX_0\Vert}{2K}\left(e^{T\Vert\bL\Vert}-1\right).
\label{eqn:numerical-error}
\end{equation}
\end{theorem}

\begin{proof}
Consider a general Euler forward scheme for our initial problem 
\begin{equation}
\begin{gathered}
    \hat\bX^{(k+1)} = \hat\bX^{(k)} -h\bL\hat\bX_t,\quad k=0,1,\dots,K-1,\quad \bX^{(0)}=\bX,
\end{gathered}
    \label{eqn:euler-forward}
\end{equation}
where $\hat\bX^{(k)}$ denotes the approximated $\bX$ at step $k$, $h$ denotes the step size and the terminal time $T=Kh$. We denote the global error at step $k$ as
\begin{equation}
    \be_k=\bX^{(k)} - \hat\bX^{(k)},
\end{equation}
and the truncation error of the Euler forward finite difference (Eqn. \eqref{eqn:euler-forward}) at step $k$ as 
\begin{equation}
    \bT^{(k)}=\frac{\bX^{(k+1)}-\bX^{(k)}}{h}+\bL\bX^{(k)}.
    \label{eqn:truncation-error-definition}
\end{equation}

We continue by noting that Eqn. \eqref{eqn:truncation-error-definition} can be written equivalently as 
\begin{equation}
    \bX^{(k+1)}=\bX^{(k)} + h\left(\bT^{(k)}-\bL\bX^{(k)}\right).
    \label{eqn:exact-update-with-truncation-error}
\end{equation}
Taking the difference of Eqn. \eqref{eqn:exact-update-with-truncation-error} and \eqref{eqn:euler-forward}, we have
\begin{equation}
    \be^{(k+1)} = (1-h\bL)\be^{(k)} + h \bT^{(k)},
\end{equation}
whose norm can be upper bounded as 
\begin{equation}
    \left\Vert\be^{(k+1)}\right\Vert \leq (1+h\Vert\bL\Vert)\left\Vert\be^{(k)}\right\Vert + h \left\Vert\bT^{(k)}\right\Vert.
\end{equation}
Let $M=\max_{0\leq k\leq K-1}\Vert\bT^{(k)}\Vert$ be the upper bound on global truncation error, we have
\begin{equation}
    \left\Vert\be^{(k+1)}\right\Vert \leq (1+h\Vert\bL\Vert)\left\Vert\be^{(k)}\right\Vert + hM.
\end{equation}
By induction, and noting that $1+h\Vert\bL\Vert\leq e^{h\Vert\bL\Vert}$ and $\be^{(0)}=\bX^{(0)}-\hat\bX^{(0)}=\bzero$, we have
\begin{equation}
\begin{aligned}
\left\Vert\be^{(K)}\right\Vert &\leq \frac{M}{\Vert\bL\Vert}\left[(1+h \Vert\bL\Vert)^{n}-1\right] \leq\frac{M}{\Vert\bL\Vert}\left(e^{Kh\Vert\bL\Vert}-1\right).
\end{aligned}
\end{equation}
Now we note that $\frac{d\bX^{(k)}}{dt}=-\bL\bX^{(k)}$ and applying Taylor's theorem, there exists $\delta\in[nh,(k+1)h]$ such that the truncation error $\bT^{(k)}$ in Eqn. \eqref{eqn:truncation-error-definition} follows
\begin{equation}
    \bT^{(k)}=\frac{1}{2h}\bL^2\bX_\delta.
\end{equation}
Thus the truncation error can be bounded by 
\begin{equation}
    \left\Vert\bT^{(k)}\right\Vert=\frac{1}{2h}\Vert\bL\Vert^2\Vert\bX_\delta\Vert
    \leq\frac{1}{2h}\Vert\bL\Vert^2\Vert\bX_0\Vert,
    \label{eqn:truncation-error-upper-bound}
\end{equation}
because 
\begin{equation}
 \Vert\bX_\delta\Vert=\left\Vert e^{-\delta\bL}\bX_0\right\Vert\leq\left\Vert\bX_0\right\Vert,\ \forall\delta\geq0.   
\end{equation}
Together with the fact $T=Kh$, we have
\begin{equation}
\begin{aligned}
\left\Vert\be^{(K)}\right\Vert &\leq \frac{\Vert\bL\Vert^2\Vert\bX_0\Vert}{2h\Vert\bL\Vert}\left(e^{Kh\Vert\bL\Vert}-1\right) =\frac{T\Vert\bL\Vert\Vert\bX_0\Vert}{2K}\left(e^{T\Vert\bL\Vert}-1\right),
\end{aligned}
\end{equation}
which completes the proof.
\end{proof}

\subsection{Proof of Theorem 3}

For the ground-truth data generation process
\begin{gather}
\bY=\bX_c\bW_c+ \sigma_y\bvarepsilon_y,\ \bvarepsilon_y\sim\cN(\bzero,\bI); \label{eqn:example-ground-truth-generation}
\end{gather}
together with the data corruption process,
\begin{equation}
\frac{d\widetilde\bX_t}{dt}=\bL\widetilde\bX_t, 
~\text{ where }\widetilde\bX_0=\bX_c ~\text{ and }~ \widetilde\bX_{T^*}=\bX. 
\label{eqn:example-inverse-heat-equation}
\end{equation}
and the final state $\bX$ denote the observed data. Then,
we have the following bound its population risks.
\begin{theorem}
Denote the population risk of the ground truth regression problem with weight $\bW$ as 
\begin{equation}
  R(\bW)=\bbE_{p(\bX_c,\bY)}\left\Vert \bY-\bX_c\bW\right\Vert^2.
\end{equation}
and that of the corrupted regression problem as 
\begin{equation}
  \hat R(\bW)=\bbE_{p(\hat\bX,\bY)}\left\Vert \bY-[\bS^{(\hat T/K)}]^K\bX\bW\right\Vert^2.
\end{equation}
Supposing that $\bbE\Vert\bX_c\Vert^2=M<\infty$, we have the following upper bound on the latter risk:
\begin{align*}
 \hat R(\bW)\leq R(\bW)+ 2\left\Vert\bW\right\Vert^2\left(
 \bbE\left\Vert\be^{(K)}_{\hat T}\right\Vert^2
 +M\left\Vert e^{T^\star\bL}\right\Vert^2\cdot\left\Vert e^{-T^\star\bL}-e^{-\hat T\bL} \right\Vert^2
 \right). 
\end{align*}
\end{theorem}

\begin{proof}
Given the fact that $\bX_c=\tilde \bX_0=e^{-T^*\bL}\bX$, we can decompose the corrupted population risk as follows
\begin{equation}
\begin{aligned}
&\hat R(\bW)=\bbE_{p(\hat\bX,\bY)}\left\Vert \bY-\left[\bS^{(\hat T/K)}\right]^K\bX\bW\right\Vert^2 \\
=&\bbE_{p(\bX,\bY)}\left\Vert\bY-\bX_c\bW+\left(e^{-T^\star\bL}-\left[\bS^{(\hat T/K)}\right]^K\right)\bX\bW\right\Vert^2\\
\leq&\bbE_{p(\bX,\bY)}\left\Vert\bY-\bX_c\bW\right\Vert^2+
+\left\Vert\bW\right\Vert^2\bbE_{p(\bX,\bY)}\left\Vert\left(\left[e^{-\hat T\bL}-\bS^{(\hat T/K)}\right]^K\right)\bX+\left(e^{-T^\star\bL}-e^{-\hat T\bL}\right)\bX\right\Vert^2\\
\leq&\bbE_{p(\bX,\bY)}\left\Vert\bY-\bX_c\bW\right\Vert^2
+\left\Vert\bW\right\Vert^2\bbE_{p(\bX,\bY)}\left\Vert\be^{(K)}_{\hat T}+\left(e^{-T^\star\bL}-e^{-\hat T\bL}\right)e^{T^\star\bL}\bX_c\right\Vert^2\\
\leq&R(\bW)+2\left\Vert\bW\right\Vert^2\left(\bbE\left\Vert\be_{\hat T}^{(K)}\right\Vert^2 + M\left\Vert e^{T^\star\bL} \right\Vert^2\left\Vert e^{-T^\star\bL}-e^{-\hat T\bL}\right\Vert^2\right),
\end{aligned}
\end{equation}
which completes the proof.
\end{proof}

\section{Further Comparison of SIGN and DGC}
\label{sec:appendix-sign}

Here, we provide a more detailed comparison of DGC and SIGN \cite{rossi2020sign}. In particular, the SIGN model is 
$$
\begin{aligned}
Y=\xi(Z\Omega),\quad  Z=\sigma([X\Theta_0,A_1X\Theta_1,\dots,A_rX\Theta_r]),
\end{aligned}
$$
where $\sigma,\xi$ are nonlinearities, $A_k=A^k$ is the $k$-hop propagation matrix, and $\Theta_k,\Omega$ are weight matrices. Our DGC-Euler model takes the form 
$$
 Y=\xi(X^{(K)}\Omega),\quad X^{(k)}=(1-T/K)X^{(k-1)}+(T/K)\cdot AX^{(k-1)},\quad k=2,\dots,K.
$$

The two models 1) both apply all feature propagation before the classification model and so that can 2) both pre-process the propagation matrix and save it for later training.  Nevertheless, there are several critical differences between SIGN and DGC:
\begin{itemize}
\item \textbf{Linear v.s. Nonlinear.} DGC is a linear model, while SIGN is nonlinear.
\item \textbf{Multi-scale SGC (SIGN) v.s. single-scale continuous diffusion (DGC).} SIGN is a multi-scale method that extracts every possible scale ($A^rX, r=0,1,\dots$) for feature propagation. Thus, SIGN resembles a multi-scale SGC, but still inherits some of the limitations of SGC, e.g. a fixed step size $\Delta t=1$. On the contrary, the goal of DGC is to find the optimal tradeoff between under-smoothing and over-smoothing with a flexible choice of $T$ (real-numbered) and fine-grained integration ($K$), so it only uses a single-scale propagation kernel (Eq. 12). 
\item \textbf{Model Size.} As a result, the model size of  (linear) SIGN is proportional to the number of scales $r$, while the model size of DGC is independent of $T$ and $K$.
\end{itemize}
Overall, we can see that the two are closely related. Below, we further compare DGC with SIGN in terms of both their performance as well as their computational efficiency.

\subsection{Fine-grained Performance Comparison}

To take a closer look at the difference between the two methods, we compare the two methods with the same terminal time $T$. 

\textbf{Setup.} We conduct the experiments on the Cora dataset (semi-supervised). We re-produce SIGN  as it has not reported results on these datasets. For comparison, we follow the same protocol of SGC, using the same optimizer, learning rate, training epochs; and (automatically) tune the weight decay and propagation steps ($K$ or $r$) at each terminal time $T$. 
\begin{table}[h]\centering
\caption{Comparison of test accuracy (\%) with different time $T$ ranging from $1$ to $6$.}
\begin{tabular}{lccccccc}
\toprule
Methods    & 1    & 2    & 3    & 4    & 5    & 5.3  & 6    \\ \midrule
SGC  & 72.4 & 80.5 & 79.2 & 81.0 & 78.8 & N/A  & 80.5 \\
SIGN & 72.4 & 77.3 & 78.9 & 80.6 & 81.3 & N/A  & 81.7 \\
DGC  & 78.0 & 81.9 & 82.5 & 83.0 & 83.1 & 83.3 & 81.5 \\ \bottomrule
\end{tabular}
\label{tab:sign-T}
\end{table}

\textbf{Results.} As shown in Table \ref{tab:sign-T}, we have the following findings:
\begin{itemize}
\item DGC still outperforms SIGN for $T\leq 5$, while being slightly worse at $T=6$ due to over-smoothing;
\item DGC can flexibly choose a real-numbered terminal time, e.g., $T=5.3$, to find the best tradeoff between under-smoothing and over-smoothing ($83.3$ acc), while the terminal time of SIGN and SGC has to be an integer;
\item Single-stage methods (SGC \& DGC) have bigger advantages at earlier time, while SIGN can surpass SGC at later stages by aggregating multi-scale information.
\end{itemize}
The empirical results show that although useful, multi-scale techniques cannot fully solve the limitations of SGC, while DGC can overcome these limitations by decoupling $T$ and $K$.

\subsection{Computation Time}

Here, we further compare the explicit training time. We also include SIGN with a different number of scales ($r$) as a baseline. The experiments are conducted on the same platform.

\begin{table}[h]\centering
\caption{Comparison of training time on the Pubmed dataset.}
\begin{tabular}{lccc}
\toprule
Method          & Preprocessing Time & Training Time & Total Time \\ \midrule
SGC / DGC ($K=2$) & 3.8 ms             & 61.5 ms       & 65.3 ms    \\
SIGN ($r=2$)      & 5.9 ms             & 78.7 ms       & 84.6 ms    \\
DGC ($K=100$)     & 169.2 ms           & 55.8 ms       & 225.0 ms   \\
SIGN ($r=100$)    & 2.4 s              & 106.9 ms      & 2.6 s      \\
GCN             & 0                  & 17.0 s        & 17.0 s     \\ \bottomrule
\end{tabular}
\label{tab:sign-time}
\end{table}

We note that the comparison of DGC ($K=100$) v.s. SGC ($K=2$) is merely designed to show how the extra propagation steps do not contribute much to the total time. Remarkably, it does not mean that DGC takes 100 steps at all settings. Sometimes, a few steps ($K<10$) are enough to attain the best performance. From Table \ref{tab:sign-time}, we have the following findings: 1) both SIGN and DGC can be much faster than GCN by pre-processing the propagation kernels; 2) DGC is still faster than SIGN with the same propagation steps by being a single-scale method.

\begin{table}[h]\centering
\caption{Comparison of training time on the PubMed dataset.}
\begin{tabular}{lccccc}
\toprule
$K$ or $r$ & 2      & 10      & 20      & 50       & 100      \\ \midrule
DGC       & 3.8 ms & 17.5 ms & 34.9 ms & 86.5 ms  & 169.2 ms \\
SIGN      & 5.9 ms & 20.5 ms & 51.5 ms & 162.3 ms & 2.4 s   \\ \bottomrule
\end{tabular}
\label{tab:sign-time-K}
\end{table}

To better understand the difference in computation time between DGC and SIGN, we note that SIGN is a multi-scale method, and it stores every intermediate scale of $A$, i.e., $[A^1,A^2,\dots,A^r]$ in the preprocessing stage with a memory complexity $O(r)$. On the contrary, DGC is a single-scale method and only needs to store one single propagation matrix $S^{(K)}$, so its memory complexity is $O(1)$. 
In modern deep learning frameworks (PyTorch in our case), it takes time to keep expanding the working GPU memory due to the copy operation. As a result, when $r$ or $K$ is very large (like $100$), SIGN can be much more memory-intensive than SGC ($100\times$). This results in a large difference in the preprocessing time between SIGN and DGC, as shown above in Table \ref{tab:sign-time-K}. There might be some tricky ways to optimize the pipeline, but here we stick to the vanilla (also official) implementation for a fair comparison.

\end{document}